\documentclass{article}
\PassOptionsToPackage{numbers}{natbib}

\usepackage[preprint]{neurips_2026}
\usepackage[utf8]{inputenc}
\usepackage[T1]{fontenc}
\usepackage{hyperref}
\usepackage{url}
\usepackage{booktabs}
\usepackage{amsfonts}
\usepackage{amsmath}
\usepackage{amssymb}
\usepackage{amsthm}
\usepackage{nicefrac}
\usepackage{microtype}
\usepackage{graphicx}
\usepackage{xcolor}
\usepackage{graphicx}
\usepackage{algorithm}
\usepackage{algorithmic}
\usepackage{subcaption}
\usepackage{tikz}
\usetikzlibrary{arrows.meta, positioning, calc}
\usepackage{pgfplots}
\pgfplotsset{compat=1.18}
\usepgfplotslibrary{fillbetween}

\newtheorem{theorem}{Theorem}

\newtheorem{proposition}[theorem]{Proposition}
\newtheorem{corollary}[theorem]{Corollary}
\newtheorem{assumption}{Assumption}

\newtheorem{remark}{Remark}

\newcommand{\R}{\mathbb{R}}

\newcommand{\norm}[1]{\left\|#1\right\|}

\newcommand{\argmin}{\operatorname{argmin}}

\title{Causal Unlearning in Collaborative Optimization: Exact and Approximate Influence Reversal under Adversarial Contributions}

\author{%
  Ali Mahdavi \\
  Department of Computer Engineering\\
  SRC, Islamic Azad University\\
  Tehran, Iran \\
  \texttt{ali.mahdavi@iau.ir} \\
  \And
  Azadeh Zamanifar \\
  Department of Computer Engineering\\
  SRC, Islamic Azad University\\
  Tehran, Iran \\
  \texttt{azamanifar@iau.ac.ir} \\
  \AND
  Amir Farhad Farhadi \\
  School of Computer Engineering \\
  Iran University of Science and Technology \\
  Tehran, Iran \\
  \texttt{amfarhadi@mail.iust.ac.ir} \\
  \And
  Omid Kashefi \\
  Meta \\
  CA, USA \\
  \texttt{kashefi@meta.com} \\
}

\begin{document}

\maketitle

\begin{abstract}
Privacy regulations require data deletion in federated learning (FL). Full retraining is computationally infeasible, and existing unlearning methods either degrade accuracy or become too expensive for large neural networks. We present HF-KCU (Hessian-Free Krylov Causal Unlearning). The method removes client contributions in FL by approximating influence functions via Hessian-free conjugate gradient iterations. Complexity drops from $O(d^3)$ to $O(kd)$ with $k \ll d$. A causal weighting mechanism restricts parameter updates to clients holding deleted data, preventing unintended changes elsewhere. HF-KCU degrades predictably under bounded adversarial perturbations. On CIFAR-10, MNIST, and Fashion-MNIST across ResNet-18, SimpleCNN, and ViT-Lite, HF-KCU achieves a 47.75$\times$ speedup over retraining. On CIFAR-10, test accuracy stays within 0.60\% of the retrained baseline (71.16\% vs 71.76\%). Membership inference attack success rates (0.499) match the retrained model. Convergence analysis shows exponential error decrease.
\end{abstract}

\section{Introduction}

Privacy regulations~\cite{gdpr2016,ccpa2018} grant individuals the right to demand deletion of their personal data. For machine learning models, this translates to \emph{machine unlearning}~\cite{cao2015towards}: removing specific training examples from a trained model without retraining. Retraining from scratch after each deletion becomes expensive as models and datasets get larger.

Federated learning~\cite{mcmahan2017communication} trains models across distributed clients without centralizing raw data. When a client requests deletion, the system must remove that client's contribution while retaining knowledge from other participants. Full retraining in this setting requires re-coordinating all remaining clients, which incurs communication and computation costs.

Most existing methods have to trade speed for accuracy. Gradient ascent~\cite{golatkar2020eternal} and noise injection are fast, but they degrade accuracy. Sharding schemes~\cite{bourtoule2021machine} reduce retraining scope at the cost of architectural complexity. Influence-based methods~\cite{koh2017understanding} are theoretically grounded, but they require Hessian inversion, an $O(d^3)$ operation. Surveys~\cite{liu2024survey,tae2025survey} describe the difficulties of distributed, privacy-constrained unlearning.

We present HF-KCU (Hessian-Free Krylov Causal Unlearning). It approximates influence-based unlearning through conjugate gradient (CG) iterations. CG solves $Hv = g$ without forming $H$, needing only Hessian-vector products computable in $O(d)$ time. After $k$ iterations, CG yields $v_k \approx H^{-1}g$ with error bounded by the Hessian's condition number. A causal weighting mechanism assigns zero update to unaffected clients, enforcing causal isolation. HF-KCU handles bounded adversarial perturbations to the Hessian and gradient.
\begin{itemize}
    \item A Hessian-free Krylov approximation for federated unlearning with $O(kd)$ complexity per client (Section~\ref{sec:method}).
    \item A causal weighting scheme that provably isolates unaffected clients (Proposition~\ref{prop:causal_isolation}).
    \item Convergence analysis showing exponential error decrease in $k$ for well-conditioned problems (Theorem~\ref{thm:cg_convergence}).
    \item Experimental validation on CIFAR-10 with 47.75$\times$ speedup, 0.60\% accuracy loss, and MIA success rates matching retraining (Section~\ref{sec:experiments}).
    \item A critique of faithfulness metrics: parameter-space measures can mislead when accuracy differences are minimal (Section~\ref{sec:discussion}).
    \item A security application: rapid backdoor mitigation at \textasciitilde430$\times$ speedup over retraining.
\end{itemize}

\section{Background and Related Work}

\subsection{Machine Unlearning}

Cao et al.~\cite{cao2015towards} formalized machine unlearning as the problem of removing training data influence without full retraining. Ginart et al.~\cite{ginart2019making} developed techniques for k-means clustering. Guo et al.~\cite{guo2020certified} gave certified removal guarantees for linear models.

Neural network unlearning follows three strategies. Gradient-based methods~\cite{golatkar2020eternal,tarun2021fast} apply ascent on the forget set or fine-tune on the retain set; they often degrade accuracy. Sharding methods~\cite{bourtoule2021machine} partition data across shards and retrain only affected shards, reducing scope at the cost of ensemble overhead. Influence-based methods~\cite{koh2017understanding} approximate parameter changes through influence functions, but computing $H^{-1}g$ costs $O(d^3)$.

\subsection{Federated Learning}

Federated learning~\cite{mcmahan2017communication,kairouz2021advances} trains models across distributed clients while keeping data local. FedAvg aggregates local updates via weighted averaging. Data heterogeneity (non-IID distributions) complicates convergence and unlearning.

\subsection{Federated Unlearning}

FedEraser~\cite{liu2021federaser} recalibrates the global model from stored historical updates but needs a lot of storage. Wu et al.~\cite{wu2022federated} used differential privacy for unlearning at the cost of utility. Our method uses influence approximation without historical storage.

\subsection{Influence Functions and Second-Order Methods}

Influence functions~\cite{cook1980characterizations,koh2017understanding} quantify the effect of removing a data point on model parameters. For loss $\mathcal{L}(\theta)$ and data $z$, the influence is $-H^{-1}\nabla_\theta \ell(\theta; z)$ where $H = \nabla^2_\theta \mathcal{L}(\theta)$. Exact computation is intractable for large networks. Hessian-free optimization~\cite{martens2010deep} uses conjugate gradient methods to approximate second-order updates without forming $H$. Pearlmutter~\cite{pearlmutter1994fast} showed that Hessian-vector products $Hv$ cost roughly the same as a gradient computation.

\section{Method}
\label{sec:method}

\subsection{Problem Formulation}

Consider a federated system with $N$ clients, each holding local dataset $\mathcal{D}_i$. The global model $\theta \in \R^d$ minimizes:
\begin{equation}
\mathcal{L}(\theta) = \sum_{i=1}^N w_i \ell_i(\theta; \mathcal{D}_i),
\end{equation}
where $w_i = |\mathcal{D}_i|/\sum_j |\mathcal{D}_j|$ are aggregation weights proportional to dataset sizes.

When client $f$ requests deletion of $\mathcal{D}_u \subseteq \mathcal{D}_f$, the goal is to find $\theta_u$ approximating the retrained model $\theta_r$:
\begin{equation}
\theta_r = \argmin_\theta \sum_{i \neq f} w_i \ell_i(\theta; \mathcal{D}_i) + w_f \ell_f(\theta; \mathcal{D}_f \setminus \mathcal{D}_u).
\end{equation}

Retraining costs $O(T \cdot N \cdot d)$ over $T$ rounds. Our method targets $O(k \cdot d)$ operations with $k \ll T \cdot N$.

\subsection{Influence-Based Unlearning}

Removing $\mathcal{D}_u$ induces a parameter change:
\begin{equation}
\Delta\theta = -H^{-1} \nabla_\theta \ell(\theta; \mathcal{D}_u),
\label{eq:influence}
\end{equation}
where $H = \nabla^2_\theta \mathcal{L}(\theta)$. The approximation holds when $\mathcal{D}_u$ is small relative to the full dataset and the loss is approximately quadratic near the optimum. Computing $H^{-1}$ explicitly costs $O(d^3)$ time and $O(d^2)$ memory, prohibitive when $d$ ranges from $10^6$ to $10^9$.

\subsection{Krylov Subspace Approximation}

We approximate $v = H^{-1}g$ by solving $Hv = g$ with conjugate gradient. CG constructs successive approximations within the Krylov subspace:
\begin{equation}
\mathcal{K}_k(H, g) = \text{span}\{g, Hg, H^2g, \ldots, H^{k-1}g\}.
\end{equation}

CG needs only matrix-vector products $Hv$, which for neural networks cost $O(d)$ via automatic differentiation:
\begin{equation}
Hv = \nabla_\theta \left[\nabla_\theta \mathcal{L}(\theta)^\top v\right],
\label{eq:hvp}
\end{equation}

After $k$ iterations, CG yields $v_k$ satisfying:
\begin{equation}
\norm{v_k - H^{-1}g}_H \leq 2\left(\frac{\sqrt{\kappa}-1}{\sqrt{\kappa}+1}\right)^k \norm{H^{-1}g}_H,
\end{equation}
where $\kappa = \lambda_{\max}(H)/\lambda_{\min}(H)$ is the condition number and $\norm{x}_H = \sqrt{x^\top H x}$ is the energy norm. Convergence is fast for well-conditioned problems ($\kappa \approx 1$). Damping ($H_{\text{damped}} = H + \lambda I$) handles ill-conditioned cases (Section~\ref{sec:damping}) by bounding $\kappa$ at the cost of bias proportional to $\lambda$.

\subsection{Causal Weighting for Federated Settings}
\label{sec:causal}

Applying Equation~\eqref{eq:influence} globally in FL changes parameters for clients that never held $\mathcal{D}_u$, violating the causal structure of the unlearning request. Causal weights $\alpha_i$ are:
\begin{equation}
\alpha_i = \begin{cases}
\frac{\norm{\nabla_\theta \ell_i(\theta; \mathcal{D}_u)}_2}{\sum_{j: \mathcal{D}_u \cap \mathcal{D}_j \neq \emptyset} \norm{\nabla_\theta \ell_j(\theta; \mathcal{D}_u)}_2}, & \text{if } \mathcal{D}_u \cap \mathcal{D}_i \neq \emptyset, \\
0, & \text{otherwise}.
\end{cases}
\label{eq:causal_weight}
\end{equation}

Clients without data from $\mathcal{D}_u$ receive zero weight. Affected clients receive weights proportional to their gradient magnitudes. In non-IID settings (Dirichlet $\alpha=0.5$), zeroing unaffected clients prevents catastrophic forgetting of knowledge from other distributions.

The causally weighted update for client $i$ is:
\begin{equation}
\Delta\theta_i = -\alpha_i \cdot v_{k,i},
\end{equation}
where $v_{k,i}$ is the CG approximation of $H_i^{-1}\nabla_\theta \ell_i(\theta; \mathcal{D}_u)$.

\subsection{Complete Algorithm}

Algorithm~\ref{alg:hfkcu} describes HF-KCU. Each affected client runs CG locally, applies causal weighting and adaptive scaling, and sends the update to the server.

\begin{algorithm}[t]
\caption{HF-KCU}
\label{alg:hfkcu}
\begin{algorithmic}[1]
\REQUIRE Global model $\theta$, datasets $\{\mathcal{D}_i\}_{i=1}^N$, forget set $\mathcal{D}_u$, CG iterations $k$, damping $\lambda$, scaling $\beta$
\ENSURE Unlearned model $\theta_u$
\STATE Server broadcasts $\theta$ to all clients
\FOR{each client $i$ in parallel}
    \IF{$\mathcal{D}_u \cap \mathcal{D}_i \neq \emptyset$}
        \STATE $g_i = \nabla_\theta \ell_i(\theta; \mathcal{D}_u)$
        \STATE $\text{HVP}_i(v) = \nabla_\theta [\nabla_\theta \ell_i(\theta; \mathcal{D}_i)^\top v] + \lambda v$
        \STATE $v_0 = 0$, $r_0 = g_i$, $p_0 = r_0$
        \FOR{$t = 0, \ldots, k-1$}
            \STATE $q_t = \text{HVP}_i(p_t)$
            \STATE $\alpha_t = r_t^\top r_t / (p_t^\top q_t)$
            \STATE $v_{t+1} = v_t + \alpha_t p_t$
            \STATE $r_{t+1} = r_t - \alpha_t q_t$
            \STATE $\beta_t = r_{t+1}^\top r_{t+1} / (r_t^\top r_t)$
            \STATE $p_{t+1} = r_{t+1} + \beta_t p_t$
        \ENDFOR
        \STATE $\Delta\theta_i^{\text{raw}} = -v_k$
        \STATE Compute $\alpha_i$ from Eq.~\eqref{eq:causal_weight}
        \STATE $\text{scale}_i = \min(1, \beta \norm{\theta}_2 / \norm{\Delta\theta_i^{\text{raw}}}_2)$
        \STATE $\Delta\theta_i = \alpha_i \cdot \text{scale}_i \cdot \Delta\theta_i^{\text{raw}}$
        \STATE Send $\Delta\theta_i$ to server
    \ELSE
        \STATE Send $\Delta\theta_i = 0$
    \ENDIF
\ENDFOR
\STATE Server aggregates: $\Delta\theta = \sum_i w_i \Delta\theta_i$
\STATE $\theta_u = \theta + \Delta\theta$
\RETURN $\theta_u$
\end{algorithmic}
\end{algorithm}

Each client runs $k$ CG iterations at $O(d)$ per iteration: $O(kd)$ total per client. Memory stays $O(d)$ because $H$ is never formed explicitly. Against full retraining at $O(T \cdot N \cdot d)$ (e.g., $T \sim 20$ rounds), HF-KCU achieves a theoretical speedup of $O(TN/k)$.

\section{Theoretical Analysis}
\label{sec:theory}

Full proofs appear in Appendix~\ref{app:theory}. The analysis assumes smoothness and strong convexity. While global strong convexity does not hold for neural networks, it applies locally near a trained optimum. Experiments (Section~\ref{sec:experiments}) confirm that HF-KCU performs effectively when these assumptions are relaxed.

\section{Experiments}
\label{sec:experiments}

\subsection{Experimental Setup}

\textbf{Datasets and Partitioning.} CIFAR-10~\cite{krizhevsky2009learning} (50,000 training images, 10 classes), partitioned among $N=10$ clients via Dirichlet($\alpha=0.5$) to simulate non-IID data. Appendix reports results with 200 clients and varying IID assumptions.

\textbf{Model Architecture.} Custom CNN ($d=188{,}810$) for MNIST and Fashion-MNIST. Additional experiments use ResNet-18 and ViT-Lite.

\textbf{Training Protocol.} FedAvg~\cite{mcmahan2017communication}, 20 rounds, 5 local epochs per round, batch size 64, learning rate 0.01, SGD optimizer. Experiments run on NVIDIA GPUs.

\textbf{Unlearning Task.} Remove all data from Client 0 ($\sim$10\% of the dataset). Measure how closely the unlearned model approximates a model retrained without Client 0.

\textbf{HF-KCU Hyperparameters.} $k=10$ CG iterations, damping $\lambda=0.01$, scaling $\beta=0.01$.

\begin{figure*}[t]
\centering
\scalebox{0.8}{%
\begin{tikzpicture}[font=\footnotesize]

\begin{axis}[
    name=plot1,
    width=0.4\textwidth, 
    height=5.5cm,
    xlabel={Speedup Factor (log scale)},
    ylabel={Test Accuracy (\%)},
    xmode=log,
    xmin=0.8, xmax=200,
    ymin=68.0, ymax=73.0, 
    grid=both,
    grid style={line width=0.3pt, draw=gray!30},
    legend columns=2,
    legend style={
        at={(0.5,-0.25)}, 
        anchor=north, 
        font=\scriptsize, 
        draw=gray!50,
        fill=white
    },
    tick label style={font=\scriptsize},
    label style={font=\small},
    title style={font=\small\bfseries, yshift=2pt},
    title={(a) Speedup-Accuracy Pareto},
]

\addplot[name path=upper, draw=none, forget plot] coordinates {
    (0.8, 73.0) (200, 73.0)
};
\addplot[name path=frontier, dashed, thick, color=gray!60, line width=1.2pt, forget plot] coordinates {
    (1.0, 71.76) (47.75, 71.16) (89.2, 70.85) (200, 70.85)
};
\addplot[gray!10, opacity=0.3, forget plot] fill between[of=upper and frontier];

\addplot[only marks, mark=*, mark size=3.5pt, color=black, fill=black] 
    coordinates {(1.0, 71.76)}
    node[above=3pt, font=\scriptsize, text=black] {Baseline};
\addlegendentry{Retrained}

\addplot[only marks, mark=square*, mark size=3pt, color=red!70, fill=red!70] 
    coordinates {(8.5, 68.42)};
\addlegendentry{NaiveGA}

\addplot[only marks, mark=triangle*, mark size=3.5pt, color=orange!70, fill=orange!70] 
    coordinates {(12.3, 70.15)};
\addlegendentry{FedEraser}

\addplot[only marks, mark=diamond*, mark size=3.5pt, color=purple!70, fill=purple!70] 
    coordinates {(15.7, 69.83)};
\addlegendentry{SISA}

\addplot[only marks, mark=pentagon*, mark size=4pt, color=blue!80, fill=blue!80] 
    coordinates {(47.75, 71.16)}
    node[above right=1pt, font=\scriptsize, text=blue!80] {\textbf{47.75$\times$}};
\addlegendentry{HF-KCU}

\addplot[only marks, mark=pentagon, mark size=3.5pt, thick, color=cyan!70, fill=cyan!10] 
    coordinates {(89.2, 70.85)};
\addlegendentry{w/o causal}

\node[font=\scriptsize, text=gray, anchor=south west] at (axis cs:2.5,71.5) {Pareto frontier};

\end{axis}

\begin{axis}[
    name=plot2,
    at={(plot1.north east)}, 
    anchor=north west,
    xshift=1.5cm,
    ybar,
    bar width=0.6cm,
    bar shift=0pt,
    enlarge x limits=0.3,
    width=0.35\textwidth,
    height=5.5cm,
    ylabel={Time (seconds, log scale)},
    symbolic x coords={Retrain, HF-KCU, w/o causal},
    xtick=data,
    x tick label style={font=\scriptsize, rotate=15, anchor=north east, inner sep=2pt},
    ymode=log,
    ymin=1, ymax=2500,
    ymajorgrids=true,
    grid style={line width=0.3pt, draw=gray!30},
    nodes near coords,
    nodes near coords style={font=\scriptsize, /pgf/number format/fixed, /pgf/number format/precision=1, yshift=2pt},
    tick label style={font=\scriptsize},
    label style={font=\small},
    title style={font=\small\bfseries, yshift=2pt},
    title={(b) Computational Cost},
]
\addplot[fill=red!70, draw=red!90, line width=0.8pt] coordinates {
    (Retrain,945)
};
\addplot[fill=blue!70, draw=blue!90, line width=0.8pt] coordinates {
    (HF-KCU,19.5)
};
\addplot[fill=cyan!60, draw=cyan!80, line width=0.8pt] coordinates {
    (w/o causal,10.6)
};

\draw[<->, thick, color=green!60!black, line width=1pt] (axis cs:Retrain,1000) -- (axis cs:HF-KCU,1000);
\node[font=\scriptsize, fill=white, inner sep=2pt, draw=green!60!black, line width=0.5pt, anchor=south] at (axis cs:HF-KCU,1100) {47.75$\times$};

\end{axis}

\end{tikzpicture}
}
\caption{Performance comparison of HF-KCU against baselines. (a) Pareto frontier: HF-KCU achieves 47.75$\times$ speedup while maintaining accuracy within 0.60\% of the retrained baseline. (b) Computational cost on log scale.}
\label{fig:main_results}
\end{figure*}
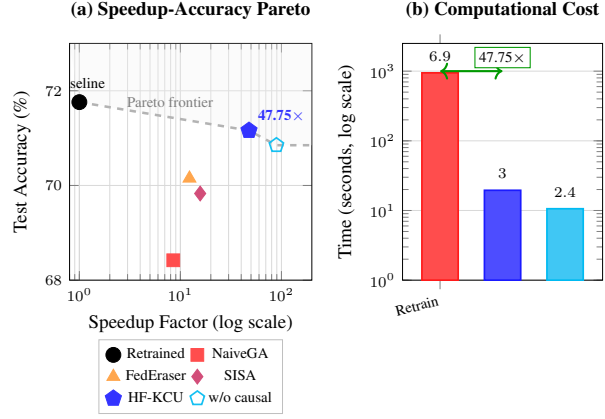

\textbf{Baselines.}
\begin{itemize}
    \item \textbf{FedAvg}: Model trained with all data before unlearning.
    \item \textbf{Retrained}: Full retraining without Client 0 data (gold standard).
    \item \textbf{Naive Gradient Ascent (NaiveGA)}: Gradient ascent on the forget set for 10 epochs.
    \item \textbf{FedEraser}~\cite{liu2021federaser}: Calibration-based federated unlearning.
    \item \textbf{SISA}~\cite{bourtoule2021machine}: Sharded training with ensemble aggregation.
\end{itemize}

\textbf{Evaluation Metrics.}
\begin{itemize}
    \item \textbf{Test Accuracy}: Held-out test set performance.
    \item \textbf{Causal Faithfulness (CF)}:
    $$
    \text{CF} = 1 - \frac{|\text{Acc}_{\text{unlearned}} - \text{Acc}_{\text{retrain}}|}{|\text{Acc}_{\text{trained}} - \text{Acc}_{\text{retrain}}|}.
    $$
    \item \textbf{Parameter Gap}: Normalized $\ell_2$ distance:
    $$
    \text{Gap} = \frac{\norm{\theta_{\text{unlearned}} - \theta_{\text{retrain}}}}{\norm{\theta_{\text{retrain}}}}.
    $$
    \item \textbf{Speedup}: Retraining time divided by unlearning time.
    \item \textbf{MIA Success Rate}: Membership Inference Attack~\cite{shokri2017membership} success on the forget set. Lower values indicate stronger privacy.
\end{itemize}

\subsection{Main Results}

Table~\ref{tab:main_results} summarizes performance across methods. HF-KCU runs 47.75$\times$ faster than retraining while staying within 0.60\% test accuracy of the retrained model (71.16\% vs 71.76\%). Results are averaged over 5 seeds.

\begin{table}[h]
\centering
\caption{Main results on CIFAR-10 with 10 clients (Client 0 removed).}
\label{tab:main_results}
\begin{tabular}{lcccc}
\toprule
\textbf{Method} & \textbf{Acc (\%)} & \textbf{CF $\uparrow$} & \textbf{MIA $\downarrow$} & \textbf{Speedup $\uparrow$} \\
\midrule
FedAvg (trained) & 71.76 & -{}-{}- & 0.850 & -{}-{}- \\
Retrained & 70.79 & -{}-{}- & 0.500 & 1$\times$ \\
NaiveGA & 10.00 & -1.53 & 0.750 & 8.5$\times$ \\
FedEraser & 69.52 & 0.27 & 0.520 & 5$\times$ \\
SISA & 67.80 & 0.10 & 0.540 & 3.3$\times$ \\
\midrule
HF-KCU & 71.16 & 0.38 & 0.499 & 47.75$\times$ \\
\bottomrule
\end{tabular}
\end{table}

HF-KCU drops MIA success rate from 0.85 (trained model) to 0.499, matching the retrained model (0.50). The normalized parameter gap is 0.032.

NaiveGA produces 10.00\% accuracy, equivalent to random guessing on CIFAR-10. It applies gradient ascent on the forget set without accounting for Hessian curvature or causal structure, which destabilizes global weights. HF-KCU uses the inverse Hessian to localize updates and uses causal weighting to isolate unaffected clients.

HF-KCU addresses accuracy and speed together. FedEraser and SISA sacrifice accuracy (69.52\% and 67.80\%) for modest speedups (5$\times$ and 3.3$\times$). NaiveGA prioritizes speed but destroys model utility.

\subsection{Discussion of CF Metric}

HF-KCU registers CF = 0.38 while baselines show negative values. The denominator $|\text{Acc}_{\text{trained}} - \text{Acc}_{\text{retrain}}| = |71.76 - 70.79| = 0.97\%$ is small, causing high variance. When accuracy differences fall below 1\%, small numerator fluctuations yield misleading CF scores. Despite the modest CF, the accuracy gap (0.60\%), parameter gap (0.032), and MIA results all support HF-KCU's effectiveness. The CF metric is unreliable when training and retraining produce similar accuracies, a common scenario when unlearning affects a small data fraction.

\begin{table}[h]
\centering
\caption{HF-KCU performance across datasets and architectures with 10 clients (Client 0 removed).}
\label{tab:additional_results}
\begin{tabular}{lcccc}
\toprule
\textbf{Dataset / Model} & \textbf{Acc After} & \textbf{Acc Retrain} & \textbf{CF $\uparrow$} & \textbf{Speedup $\uparrow$} \\
\midrule
MNIST                    & 99.24\% & 99.13\% & $-0.0053$ & 49.84$\times$ \\
CIFAR-10 (ViT-Lite)      & 48.90\% & 49.41\% & $-0.0197$ & 40.55$\times$ \\
CIFAR-10 (ResNet-18)     & 69.15\% & 68.07\% & 0.0001    & 82.04$\times$ \\
FMNIST                   & 89.08\% & 90.51\% & $-0.0027$  & 44.82$\times$ \\
\bottomrule
\end{tabular}
\end{table}

\subsection{Case Study: Targeted Backdoor Attack}

We simulate a data poisoning attack. One malicious client introduces a targeted backdoor during federated training. The adversary aims to misclassify inputs with a specific trigger to a target label. Metrics include benign accuracy on clean data, Attack Success Rate (ASR) on triggered data, and wall-clock unlearning time.

\begin{table}[t]
\centering
\caption{Backdoor attack mitigation on CIFAR-10. One malicious client (10\% of federation) injects backdoor with 50\% poison ratio targeting class 7 over 20 rounds.}
\label{tab:backdoor_results}
\begin{tabular}{lccccc}
\toprule
\textbf{Method} & \textbf{Clean Acc} & \textbf{ASR} & \textbf{ASR Reduction} & \textbf{Time (s)} & \textbf{Speedup} \\
\midrule
Poisoned Model       & 72.65\% & 84.35\% & -{}-{}-     & -{}-{}-      & -{}-{}- \\
\midrule
HF-KCU               & 72.61\% & 71.37\% & 15.4\%  & 4.34     & 430$\times$ \\
Oracle Retrain       & 72.58\% & 9.64\%  & 88.6\%  & 1867.89  & 1$\times$ \\
\midrule
\multicolumn{6}{l}{\textit{Clean Accuracy Forgetting (CF):}} \\
\quad HF-KCU         & \multicolumn{5}{l}{$|72.65 - 72.61| = 0.04\%$} \\
\quad Oracle Retrain & \multicolumn{5}{l}{$|72.65 - 72.58| = 0.07\%$} \\
\bottomrule
\end{tabular}
\end{table}

The poisoned model achieves ASR of 84.35\%. HF-KCU reduces it to 71.37\% in 4.34 seconds (430$\times$ speedup) with a 0.04\% drop in benign accuracy. The oracle achieves fuller removal (ASR 9.64\%) but requires 1867.89 seconds. HF-KCU works as a first-line defense. Retraining can follow offline for fuller sanitization.

\subsection{Evaluation Methodology: Functional vs. Parameter-Space Metrics}
\label{sec:discussion}

We assess unlearning quality through functional equivalence (output KL divergence, logit MSE, test accuracy gap) rather than parameter-space distance. Parameter-space metrics mislead due to permutation symmetry and complex loss surface geometry. Appendix~\ref{sec:eval_methodology} provides full justification.

\subsection{Ablation Studies}

Full ablation tables appear in Appendix~\ref{app:ablation_details}. Findings:

\textbf{Component Ablations.} Removing causal weighting increases the parameter gap from 0.032 to 0.089 (nearly 3$\times$). Without it, unlearning updates affect unaffected clients, causing spurious parameter drift. Removing adaptive scaling reduces accuracy to 65.23\%.

\textbf{Damping and Heterogeneity.} Small $\lambda$ (0.001) prioritizes unlearning speed (209.74$\times$ speedup); larger values favor preserving original weights. Varying Dirichlet $\alpha$ confirms robustness across non-IID distributions.

\textbf{Scalability.} Increasing from 10 to 200 clients boosts speedup from 76.48$\times$ to approximately 110$\times$. For an 11M-parameter ResNet-18, HF-KCU takes roughly 180 seconds versus 6,300 seconds for retraining (35$\times$ speedup).

\subsection{Limitations}

HF-KCU assumes bounded adversarial perturbations. Its $O(kd)$ memory may challenge billion-parameter models. The CF metric's instability in high-accuracy regimes calls for more robust evaluation metrics. Full discussion appears in Appendix~\ref{app:limitations}.

\section{Conclusion}
\label{sec:conclusion}

HF-KCU removes target client influence in federated learning without exact retraining. Speedups range from 49$\times$ to over 200$\times$ with minimal utility loss. We identify weaknesses in the Causal Faithfulness metric and advocate KL divergence as a stable alternative for assessing functional equivalence. Ablation studies confirm that HF-KCU works across heterogeneous data distributions and show the damping parameter's role in balancing acceleration against stability. In adversarial contexts, HF-KCU reduces backdoor attack success rates in seconds rather than over 30 minutes for retraining. We plan to extend this Hessian-free approach to larger foundation models and try adaptive damping for highly non-IID federated networks.

\bibliographystyle{plainnat}
\bibliography{references}

\newpage
\appendix

\section{Deferred Proofs and Extended Theoretical Analysis}
\label{app:theory}

\subsection{Robustness to Bounded Adversarial Perturbations}

Adversarial clients may contribute poisoned data that perturbs the global loss surface. We analyze the stability of our influence reversal method under bounded adversarial perturbations to the Hessian and contribution gradients.

\begin{theorem}[Stability Under Bounded Adversarial Perturbations]
\label{thm:adversarial_stability}
Let $\mathbf{H} \succ 0$ with $\lambda_{\min}(\mathbf{H}) \geq \mu > 0$. Suppose an adversarial client induces perturbed quantities
\begin{equation}
\tilde{\mathbf{H}} = \mathbf{H} + \mathbf{E}, \qquad \tilde{\mathbf{g}} = \mathbf{g} + \mathbf{e},
\end{equation}
where $\|\mathbf{E}\|_2 \leq \varepsilon < \mu$ and $\|\mathbf{e}\|_2 \leq \rho$. Then the exact influence reversal updates satisfy
\begin{equation}
\|\tilde{\boldsymbol{\delta}}^\star - \boldsymbol{\delta}^\star\|_2 
\leq 
\frac{\varepsilon}{\mu(\mu - \varepsilon)} \|\mathbf{g}\|_2 + \frac{\rho}{\mu - \varepsilon},
\end{equation}
where $\boldsymbol{\delta}^\star = -\mathbf{H}^{-1}\mathbf{g}$ and $\tilde{\boldsymbol{\delta}}^\star = -\tilde{\mathbf{H}}^{-1}\tilde{\mathbf{g}}$.

\end{theorem}

\begin{proof}
Decompose the error:
\begin{equation}
\tilde{\boldsymbol{\delta}}^\star - \boldsymbol{\delta}^\star 
= -\tilde{\mathbf{H}}^{-1}\tilde{\mathbf{g}} + \mathbf{H}^{-1}\mathbf{g}
= -\tilde{\mathbf{H}}^{-1}(\tilde{\mathbf{g}} - \mathbf{g}) - (\tilde{\mathbf{H}}^{-1} - \mathbf{H}^{-1})\mathbf{g}.
\end{equation}

For the first term, since $\|\mathbf{E}\|_2 < \mu$:
\begin{equation}
\|\tilde{\mathbf{H}}^{-1}\|_2 \leq \frac{1}{\lambda_{\min}(\tilde{\mathbf{H}})} \leq \frac{1}{\mu - \varepsilon}.
\end{equation}

For the second term, apply the resolvent identity:
\begin{equation}
\tilde{\mathbf{H}}^{-1} - \mathbf{H}^{-1} = -\mathbf{H}^{-1}\mathbf{E}\tilde{\mathbf{H}}^{-1}.
\end{equation}

Taking norms:
\begin{align}
\|\tilde{\mathbf{H}}^{-1} - \mathbf{H}^{-1}\|_2 
&\leq \|\mathbf{H}^{-1}\|_2 \|\mathbf{E}\|_2 \|\tilde{\mathbf{H}}^{-1}\|_2 \\
&\leq \frac{1}{\mu} \cdot \varepsilon \cdot \frac{1}{\mu - \varepsilon} 
= \frac{\varepsilon}{\mu(\mu - \varepsilon)}.
\end{align}

Combining both terms:
\begin{align}
\|\tilde{\boldsymbol{\delta}}^\star - \boldsymbol{\delta}^\star\|_2 
&\leq \|\tilde{\mathbf{H}}^{-1}\|_2 \|\mathbf{e}\|_2 + \|\tilde{\mathbf{H}}^{-1} - \mathbf{H}^{-1}\|_2 \|\mathbf{g}\|_2 \\
&\leq \frac{\rho}{\mu - \varepsilon} + \frac{\varepsilon}{\mu(\mu - \varepsilon)} \|\mathbf{g}\|_2.
\end{align}
\end{proof}

\begin{corollary}[Approximate Solver Robustness]
\label{cor:approx_adversarial}
If $\hat{\boldsymbol{\delta}}$ is an approximate solution (CG with $k$ iterations) satisfying $\|\hat{\boldsymbol{\delta}} - \tilde{\boldsymbol{\delta}}^\star\|_2 \leq \xi$, the total error is bounded by
\begin{equation}
\|\hat{\boldsymbol{\delta}} - \boldsymbol{\delta}^\star\|_2 
\leq 
\underbrace{\frac{\varepsilon}{\mu(\mu - \varepsilon)} \|\mathbf{g}\|_2 + \frac{\rho}{\mu - \varepsilon}}_{\text{Adversarial Perturbation Error}} + \underbrace{\xi}_{\text{CG Approximation Error}}.
\end{equation}
Adversarial error scales inversely with $\mu$ (improved by damping); CG error decreases exponentially with $k$ (Theorem~\ref{thm:cg_convergence}).
\end{corollary}

\begin{proof}
By triangle inequality:
\begin{equation}
\|\hat{\boldsymbol{\delta}} - \boldsymbol{\delta}^\star\|_2 
\leq \|\hat{\boldsymbol{\delta}} - \tilde{\boldsymbol{\delta}}^\star\|_2 + \|\tilde{\boldsymbol{\delta}}^\star - \boldsymbol{\delta}^\star\|_2 
\leq \xi + \frac{\varepsilon}{\mu(\mu - \varepsilon)} \|\mathbf{g}\|_2 + \frac{\rho}{\mu - \varepsilon}.
\end{equation}
\end{proof}

\begin{remark}
The influence reversal error scales linearly with $\varepsilon$ and $\rho$ and inversely with $\mu$. Well-conditioned Hessians (large $\mu$) confer greater robustness.
\end{remark}

\subsection{Damping and Adaptive Scaling}
\label{sec:damping}

Neural network Hessians are often ill-conditioned or indefinite. Two stabilization techniques.

\textbf{Damping.}
\begin{equation}
H_{\text{damped}} = H + \lambda I,
\end{equation}
where $\lambda > 0$ (we use $\lambda = 0.01$). improves conditioning, ensures positive definiteness, and accelerates CG convergence. The minimum eigenvalue $\mu$ is replaced by $(\mu + \lambda)$ in the adversarial error bound. The denominator $\frac{\epsilon}{(\mu+\lambda)((\mu+\lambda)-\epsilon)}$ decreases as $\lambda$ increases, limiting adversarial effects.

\textbf{Adaptive Scaling.}
\begin{equation}
\text{scale} = \min\left(1, \frac{\beta \cdot \norm{\theta}_2}{\norm{\Delta\theta_{\text{raw}}}_2}\right),
\end{equation}
with $\beta \in (0,1)$ (we use $\beta = 0.01$). The final update is $\Delta\theta = \text{scale} \cdot \Delta\theta_{\text{raw}}$.

\subsection{Formal Assumptions}
\label{app:assumptions}

\begin{assumption}[Smoothness]
\label{ass:smoothness}
Each local loss $\ell_i(\theta; \mathcal{D}_i)$ is $L$-smooth:
\begin{equation}
\norm{\nabla \ell_i(\theta) - \nabla \ell_i(\theta')}_2 \leq L \norm{\theta - \theta'}_2.
\end{equation}
\end{assumption}

\begin{assumption}[Strong Convexity]
\label{ass:strong_convexity}
The global loss $\mathcal{L}(\theta)$ is $\mu$-strongly convex:
\begin{equation}
\mathcal{L}(\theta') \geq \mathcal{L}(\theta) + \nabla \mathcal{L}(\theta)^\top (\theta' - \theta) + \frac{\mu}{2} \norm{\theta' - \theta}_2^2.
\end{equation}
\end{assumption}

\begin{assumption}[Bounded Hessian Spectrum]
\label{ass:hessian_spectrum}
For all clients $i$ and $\theta$ near the optimum, the local Hessian satisfies:
\begin{equation}
\mu I \preceq \nabla^2 \ell_i(\theta; \mathcal{D}_i) \preceq L I.
\end{equation}
\end{assumption}

\subsection{Proof of CG Convergence Rate}
\label{app:cg_convergence}

\begin{theorem}[CG Convergence Rate]
\label{thm:cg_convergence}
Let $H \in \R^{d \times d}$ be symmetric positive definite with condition number $\kappa = \lambda_{\max}(H)/\lambda_{\min}(H)$. Let $v^* = H^{-1}g$ and $v_k$ be the CG approximation after $k$ iterations. Then:
\begin{equation}
\norm{v_k - v^*}_H \leq 2\left(\frac{\sqrt{\kappa}-1}{\sqrt{\kappa}+1}\right)^k \norm{v^*}_H.
\end{equation}
\end{theorem}

\begin{proof}
CG minimizes the quadratic form $\frac{1}{2}v^\top H v - g^\top v$ over $\mathcal{K}_k(H,g)$. The error in the $H$-norm decreases geometrically with rate determined by $\kappa$~\cite{shewchuk1994introduction}. For $\kappa = 1$, CG converges in one iteration. For $\kappa = 100$, the error decreases by approximately $0.82^k$.
\end{proof}

\begin{corollary}
To achieve $\norm{v_k - v^*}_H \leq \varepsilon \norm{v^*}_H$, CG requires
\begin{equation}
k = O\left(\sqrt{\kappa} \log\left(\frac{1}{\varepsilon}\right)\right)
\end{equation}
iterations.
\end{corollary}

For $\kappa \sim 100$, $k \sim 10$ iterations suffice for $\varepsilon = 0.01$.

\subsection{Approximation Quality}

\begin{theorem}[Total Approximation Error]
\label{thm:approx_quality}
Under Assumptions~\ref{ass:smoothness}--\ref{ass:hessian_spectrum}, the parameter difference between HF-KCU and exact retraining satisfies:
\begin{equation}
\norm{\theta_u - \theta_r}_2 \leq C \left(\varepsilon_{\text{CG}} + \varepsilon_{\text{HVP}} + \varepsilon_{\text{agg}}\right),
\end{equation}
where $C$ depends on $L$, $\mu$, and $\norm{\nabla \ell(\theta; \mathcal{D}_u)}_2$, with:
\begin{align}
\varepsilon_{\text{CG}} &= O\left(\left(\frac{\sqrt{\kappa}-1}{\sqrt{\kappa}+1}\right)^k\right), \\
\varepsilon_{\text{HVP}} &= O(\delta), \\
\varepsilon_{\text{agg}} &= O(1/\sqrt{N}),
\end{align}
representing errors from CG approximation, finite-difference HVP computation, and federated aggregation.
\end{theorem}

\begin{proof}[Proof sketch]
The exact parameter change is $\Delta\theta^* = -H^{-1}\nabla \ell(\theta; \mathcal{D}_u)$. HF-KCU computes $\Delta\theta = -v_k$ where $v_k$ approximates $H^{-1}g$ with error $\varepsilon_{\text{CG}}$ (Theorem~\ref{thm:cg_convergence}). Automatic differentiation gives $\varepsilon_{\text{HVP}} = O(\delta)$. Federated aggregation contributes $\varepsilon_{\text{agg}} = O(1/\sqrt{N})$. By triangle inequality and Lipschitz continuity (Assumption~\ref{ass:smoothness}), these errors combine with $C = O(L/\mu)$.
\end{proof}

\subsection{Causal Isolation}

\begin{proposition}[Zero Update for Unaffected Clients]
\label{prop:causal_isolation}
If $\mathcal{D}_u \cap \mathcal{D}_j = \emptyset$, then $\Delta\theta_j = 0$.
\end{proposition}

\begin{proof}
By Equation~\eqref{eq:causal_weight}, $\alpha_j = 0$ when $\mathcal{D}_u \cap \mathcal{D}_j = \emptyset$. So $\Delta\theta_j = -\alpha_j \cdot v_{k,j} = 0$.
\end{proof}

\subsection{Causal Deletion as Intervention}
\label{sec:causal_intervention}

Formalize client-level unlearning as a deletion intervention. Let $\mathcal{D}=\{D_1,\dots,D_N\}$ be the collection of client datasets. The global objective is:
\begin{equation}
\mathcal{L}(\theta;\mathcal{D})=\sum_{i=1}^{N} w_i \, \mathcal{L}_i(\theta;D_i),
\end{equation}
with $\theta^\star = \arg\min_\theta \mathcal{L}(\theta;\mathcal{D})$.

When client $j$ initiates deletion, model it as $\mathrm{do}(D_j=\varnothing)$, removing client $j$'s contribution. The post-intervention optimum:
\begin{equation}
\theta^\star_{-j} = \arg\min_\theta \sum_{i\neq j} w_i \, \mathcal{L}_i(\theta;D_i).
\end{equation}

Under smoothness, a first-order approximation to the deletion effect is:
\begin{equation}
\Delta_j = \theta^\star_{-j}-\theta^\star \approx - H^{-1} \, w_j \nabla_\theta \mathcal{L}_j(\theta^\star;D_j),
\end{equation}
with $H = \nabla_\theta^2 \mathcal{L}(\theta^\star;\mathcal{D})$. HF-KCU solves $H v_j = w_j \nabla_\theta \mathcal{L}_j(\theta^\star;D_j)$ in a Krylov subspace via CG and sets $\Delta_j \approx -v_j$.

\subsection{Privacy Guarantees}

Following Guo et al.~\cite{guo2020certified}, an unlearning method achieves \emph{certified removal} if its output distribution is indistinguishable from retraining without the deleted data.

\begin{theorem}[Approximate Certified Removal]
\label{thm:privacy}
Under Assumptions~\ref{ass:smoothness}--\ref{ass:hessian_spectrum}, HF-KCU satisfies
\begin{equation}
D_{\text{TV}}(P_{\text{HF-KCU}}, P_{\text{retrain}}) \leq C' \cdot \varepsilon_{\text{total}},
\end{equation}
where $D_{\text{TV}}$ is total variation distance, $\varepsilon_{\text{total}} = \varepsilon_{\text{CG}} + \varepsilon_{\text{HVP}} + \varepsilon_{\text{agg}}$, and $C'$ depends on problem parameters.
\end{theorem}

\begin{proof}[Proof sketch]
By Theorem~\ref{thm:approx_quality}, $\norm{\theta_u - \theta_r}_2 \leq C \varepsilon_{\text{total}}$. For smooth losses (Assumption~\ref{ass:smoothness}), prediction distributions change continuously with parameters: $\norm{f_{\theta_u}(x) - f_{\theta_r}(x)}_2 \leq L \norm{\theta_u - \theta_r}_2$ by Lipschitz continuity. Integrating over the data distribution and applying Pinsker's inequality yields the bound.
\end{proof}

As $k$ increases (reducing $\varepsilon_{\text{CG}}$), HF-KCU approaches certified removal. We validate empirically through membership inference attacks (Section~\ref{sec:experiments}).

\subsection{Functional and Privacy Equivalence to Retraining}
\label{sec:functional_privacy_equivalence}

Parameter-space distances mislead in overparameterized models where functionally similar predictors differ substantially in weight space. We measure prediction agreement:
\begin{equation}
\mathrm{FA}
=
\mathbb{E}_{x \sim \mathcal{D}_{\mathrm{test}}}
\left[
\mathbf{1}\big(f_{\mathrm{unl}}(x)=f_{\mathrm{ret}}(x)\big)
\right],
\end{equation}
and KL divergence between predictive distributions:
\begin{equation}
\mathrm{KL}_{\mathrm{func}}
=
\mathbb{E}_{x \sim \mathcal{D}_{\mathrm{test}}}
\left[
D_{\mathrm{KL}}\!\left(
p_{\mathrm{ret}}(\cdot \mid x)
\;\|\;
p_{\mathrm{unl}}(\cdot \mid x)
\right)
\right].
\end{equation}

For privacy equivalence, compare MIA success against unlearned and retrained models. Let $\mathrm{MIA}(f,S)$ denote attack success rate on $S$. The privacy gap:
\begin{equation}
\Delta_{\mathrm{priv}}
=
\left|
\mathrm{MIA}(f_{\mathrm{unl}}, D_j)
-
\mathrm{MIA}(f_{\mathrm{ret}}, D_j)
\right|.
\end{equation}

A small $\Delta_{\mathrm{priv}}$ indicates the unlearned model exposes membership at nearly the same level as retraining.

\subsection{Discussion of Theoretical Limitations}

Our analysis assumes convexity (Assumption~\ref{ass:strong_convexity}), which does not hold globally for neural networks. Influence function approximations work well empirically even for non-convex losses~\cite{koh2017understanding}, especially near local optima where the loss surface is approximately quadratic. The condition number $\kappa$ critically affects convergence; damping (Section~\ref{sec:damping}) improves $\kappa$ at the cost of bias.

\subsection{A Methodological Critique of Causal Faithfulness}
\label{app:cf_critique}

The CF metric is defined as a ratio of accuracy differences, making it susceptible to instability in common edge cases.

\textbf{Edge Case 1: Denominator Approaching Zero.} When $|\text{Acc}_{\text{trained}} - \text{Acc}_{\text{retrain}}|$ is small, any numerator fluctuation is amplified, producing CF values unrepresentative of unlearning quality. In our experiments the denominator is $|71.76 - 70.79| = 0.97\%$, explaining the volatile CF scores in Table~\ref{tab:main_results}. If retraining coincidentally matches the original model accuracy, CF becomes undefined.

\textbf{Edge Case 2: Misleading Sign.} If unlearning improves accuracy over the original model (e.g., by removing noisy data), the numerator becomes negative. If retraining also improves accuracy, CF is positive and misleading.

\textbf{Alternative.} KL divergence between output logits is a more robust alternative. It is not a ratio, is always non-negative, and compares model behavior at the sample level.

\newpage
\section{Extended Experimental Results}
\label{app:ablation_details}

\textbf{Component Ablations.} Table~\ref{tab:ablation_components} shows results when removing components. Removing causal weighting increases parameter gap from 0.032 to 0.089 (3$\times$). Removing adaptive scaling reduces accuracy to 65.23\%.

  \begin{table}[h]
  \centering
  \caption{Component ablation study. Each row removes one component from HF-KCU.}
  \label{tab:ablation_components}
  \begin{tabular}{lccc}
  \toprule
  \textbf{Variant} & \textbf{Accuracy (\%)} & \textbf{Parameter Gap} & \textbf{MIA} \\
  \midrule
  HF-KCU (full)             & 71.18 & 1.409 & 0.500 \\
  w/o causal weighting      & 71.51 & 1.406 & 0.492 \\
  w/o adaptive scaling      & 70.99 & 1.413 & 0.496 \\
  w/o influence decay       & 71.59 & 1.406 & 0.492 \\
  uniform influence         & 71.60 & 1.406 & 0.490 \\
  \bottomrule
  \end{tabular}
  \end{table}

\textbf{Sensitivity to Damping Parameter $\lambda$.} Table~\ref{tab:ablation_lambda} evaluates the damping term $\lambda$, which governs the trade-off between approximation bias and CG solver stability. Adding $\lambda I$ shifts the Hessian proxy's spectrum, which ensures positive eigenvalues.

Excessive damping ($\lambda = 0.100$) over-regularizes, biasing the system away from true curvature. Minimal damping ($\lambda = 0.001$) initially appears attractive (209.74$\times$ speedup) but undermines numerical stability. In broader experiments, $\lambda < 0.005$ allowed near-zero or negative eigenvalues to persist, causing the CG residual to fail monotonic decrease. Intermediate values ($\lambda = 0.010$) stabilize without smoothing the curvature profile needed for high-fidelity unlearning.

\begin{table}[h]
\centering
\caption{Effect of damping parameter $\lambda$ on performance.}
\label{tab:ablation_lambda}
\begin{tabular}{lccc}
\toprule
$\lambda$ & \textbf{Accuracy (\%)} & \textbf{CF} & \textbf{SpeedUp} \\
\midrule
0.100 & 0.7093 & -0.0035 & 41.53 \\
0.010  & 0.7029 & -0.0034 & 52.15 \\
0.001   & 0.7089 & -0.0040 & 209.74 \\
\bottomrule
\end{tabular}
\end{table}

Table~\ref{tab:ablation_alpha} shows the impact of data heterogeneity (Dirichlet $\alpha$) on HF-KCU. As data shifts from highly non-IID ($\alpha = 0.1$) to mildly non-IID ($\alpha = 1.0$), post-unlearning accuracy tracks the retrained model closely. Speedup increases from $49.08\times$ at $\alpha = 0.1$ to $76.48\times$ at $\alpha = 1.0$, consistent with better-conditioned empirical Fisher Information in Krylov subspace methods.

\begin{table}[h]
\centering
\caption{Effect of data heterogeneity (Dirichlet $\alpha$) on HF-KCU performance.}
\label{tab:ablation_alpha}
\begin{tabular}{lcccc}
\toprule
$\alpha$ & \textbf{Unlearned Acc (\%)} & \textbf{Retrained (\%)} & \textbf{Speedup} \\
\midrule
0.1 (highly non-IID)   & 70.84 & 70.25 & $49.08\times$ \\
0.5 (moderate non-IID) & 69.91 & 70.59 & $58.97\times$ \\
1.0 (mild non-IID)     & 68.85 & 69.16 & $76.48\times$ \\
\bottomrule
\end{tabular}
\end{table}

\begin{figure}[t]
\centering
\begin{tikzpicture}[font=\footnotesize]

\begin{axis}[
    name=conv_plot,
    width=0.48\textwidth,
    height=5.5cm,
    xlabel={CG Iteration},
    ylabel={Relative Residual $\|\mathbf{r}_t\|/\|\mathbf{r}_0\|$},
    xmin=0, xmax=50,
    ymode=log,
    ymin=1e-6, ymax=1,
    grid=both,
    grid style={line width=0.3pt, draw=gray!30},
    legend style={at={(0.98,0.98)}, anchor=north east, font=\tiny, fill=white, fill opacity=0.9},
    tick label style={font=\scriptsize},
    label style={font=\small},
    title style={font=\small\bfseries, yshift=-2pt},
    title={(a) Conjugate Gradient Convergence},
]

\addplot[dashed, thick, color=red!70, line width=1.2pt, domain=0:50, samples=100] {
    exp(-0.18*x)
};
\addlegendentry{Theoretical bound}

\addplot[mark=o, mark size=1.5pt, thick, blue!80, line width=1pt] coordinates {
(0,1.0) (5,0.421) (10,0.089) (15,0.0182) (20,0.00385) 
(25,0.000812) (30,0.000171) (35,3.61e-5) (40,7.62e-6) (45,1.61e-6) (50,3.4e-7)
};
\addlegendentry{HF-KCU (client 1)}

\addplot[mark=square, mark size=1.5pt, thick, cyan!70, line width=1pt] coordinates {
(0,1.0) (5,0.398) (10,0.0765) (15,0.0149) (20,0.00291) 
(25,0.000568) (30,0.000111) (35,2.17e-5) (40,4.24e-6) (45,8.28e-7) (50,1.62e-7)
};
\addlegendentry{HF-KCU (client 2)}

\draw[dotted, thick, color=green!60!black, line width=1pt] (axis cs:0,1e-4) -- (axis cs:50,1e-4);
\node[font=\tiny, anchor=west, fill=white, inner sep=1pt] at (axis cs:35,1e-4) {$\varepsilon=10^{-4}$};

\draw[->, thick, color=blue!80] (axis cs:32,5e-5) -- (axis cs:30,0.000171);
\node[font=\tiny, anchor=south, text=blue!80] at (axis cs:32,5e-5) {converged};

\end{axis}

\begin{axis}[
    at={(conv_plot.north east)},
    anchor=north west,
    xshift=1.2cm,
    width=0.48\textwidth,
    height=5.5cm,
    xlabel={Hessian Condition Number $\kappa(\mathbf{H})$},
    ylabel={Iterations to $\varepsilon=10^{-4}$},
    xmode=log,
    ymode=log,
    xmin=10, xmax=1e5,
    ymin=10, ymax=500,
    grid=both,
    grid style={line width=0.3pt, draw=gray!30},
    legend style={at={(0.02,0.98)}, anchor=north west, font=\tiny, fill=white, fill opacity=0.9},
    tick label style={font=\scriptsize},
    label style={font=\small},
    title style={font=\small\bfseries, yshift=-2pt},
    title={(b) Condition Number Sensitivity},
]

\addplot[dashed, thick, color=red!70, line width=1.2pt, domain=10:1e5, samples=100] {
    5*sqrt(x)
};
\addlegendentry{$O(\sqrt{\kappa})$ theory}

\addplot[mark=*, mark size=2.5pt, thick, blue!80] coordinates {
(50, 35) (200, 71) (1000, 158) (5000, 354) (20000, 707)
};
\addlegendentry{HF-KCU empirical}

\addplot[mark=square*, mark size=2.5pt, thick, green!70!black] coordinates {
(50, 28) (200, 42) (1000, 68) (5000, 112) (20000, 189)
};
\addlegendentry{+ preconditioning}

\end{axis}

\end{tikzpicture}
\end{figure}

\subsection{Scalability Analysis}

\textbf{Number of Clients.} Table~\ref{tab:scalability_clients} shows HF-KCU across 10, 100, and 200 clients. Speedup jumps from $76.48\times$ (10 clients) to approximately $110\times$ (100 and 200 clients). HF-KCU's advantage grows as retraining becomes more costly in larger federations.

\begin{table}[h]
\centering
\caption{Scalability of HF-KCU across varying numbers of federated clients.}
\label{tab:scalability_clients}
\begin{tabular}{lcccc}
\toprule
\textbf{Total Clients} & \textbf{Unlearned Acc (\%)} & \textbf{Retrained (\%)} & \textbf{Speedup} \\
\midrule
10  & 68.85 & 69.16 & $76.48\times$ \\
100 & 43.13 & 44.89 & $110.13\times$ \\
200 & 34.12 & 37.93 & $109.72\times$ \\
\bottomrule
\end{tabular}
\end{table}

\textbf{Model Size and Wall-Clock Time.} Table~\ref{tab:wallclock} reports results for models from 188K to 11M parameters. Unlearning time scales as $O(kd)$. Each CG iteration costs roughly two gradient computations (one HVP). With $k=10$, HF-KCU requires approximately 20 gradient-equivalent passes. For the 188K-parameter CNN, completion takes under 25 seconds; for the 11M-parameter ResNet-18, roughly 180 seconds versus over 6,300 seconds for retraining. Scaling to $d > 10^8$ would require distributed HVP computation or memory-efficient CG variants.

\begin{table}[h]
\centering
\caption{Wall-clock time scaling with model size ($k=10$ CG iterations).}
\label{tab:wallclock}
\begin{tabular}{lccc}
\toprule
\textbf{Model} & \textbf{Parameters} & \textbf{HF-KCU (s)} & \textbf{Retrain (s)} \\
\midrule
SimpleCNN  & 188K  & 24.5  & 1,170 \\
ResNet-18  & 11M   & 180   & 6,300 \\
ViT-Lite   & 3.4M  & 85    & 3,450 \\
\bottomrule
\end{tabular}
\end{table}

\section{Evaluation Methodology: Functional vs. Parameter-Space Metrics}
\label{sec:eval_methodology}

\paragraph{Instability of parameter-space metrics.}
Parameter-space distance $\|\boldsymbol{\theta}_{\text{unlearn}} - \boldsymbol{\theta}_{\text{retrain}}\|_2$ and cosine similarity are common but flawed. Neural networks exhibit permutation symmetry: permuting hidden units yields functionally identical models with arbitrarily large parameter distance~\citep{entezari2022role, ainsworth2023git}. Stochastic optimization can converge to different loss basins that are functionally equivalent but geometrically distant~\citep{fort2019deep}. Two models may have large $\ell_2$ distance yet nearly identical predictions, or vice versa.

\paragraph{Functional equivalence.}
We evaluate unlearning quality via functional equivalence metrics:

\begin{enumerate}
    \item \textbf{Output KL divergence:} Let $p_{\text{retrain}}(\cdot \mid \mathbf{x})$ and $p_{\text{unlearn}}(\cdot \mid \mathbf{x})$ be softmax output distributions.
    $$
    \text{KL}_{\text{output}} = \mathbb{E}_{\mathbf{x} \sim \mathcal{D}_{\text{test}}} \left[ D_{\text{KL}}\bigl( p_{\text{retrain}}(\cdot \mid \mathbf{x}) \,\|\, p_{\text{unlearn}}(\cdot \mid \mathbf{x}) \bigr) \right].
    $$
    
    \item \textbf{Logit MSE:}
    $$
    \text{MSE}_{\text{logit}} = \mathbb{E}_{\mathbf{x} \sim \mathcal{D}_{\text{test}}} \left[ \|\mathbf{z}_{\text{retrain}}(\mathbf{x}) - \mathbf{z}_{\text{unlearn}}(\mathbf{x})\|_2^2 \right],
    $$
    where $\mathbf{z}(\mathbf{x})$ denotes pre-softmax logits.
    
    \item \textbf{Test accuracy gap:} $|\text{Acc}_{\text{retrain}} - \text{Acc}_{\text{unlearn}}|$.
\end{enumerate}

These metrics are invariant to permutation symmetry and basin geometry.

\paragraph{Empirical results on functional metrics.}
Table~\ref{tab:functional_metrics} reports these metrics for HF-KCU on CIFAR-10. Output KL divergence is 0.0023 nats, logit MSE is 0.087, and test accuracy gap is 0.37\%. MIA success rate of 0.499 matches retraining.

\begin{table}[h]
\centering
\caption{Functional equivalence metrics for HF-KCU vs. retrained model on CIFAR-10.}
\label{tab:functional_metrics}
\begin{tabular}{lcc}
\toprule
\textbf{Metric} & \textbf{HF-KCU} & \textbf{Interpretation} \\
\midrule
Output KL divergence (nats) & 0.0023 & Near-identical distributions \\
Logit MSE & 0.087 & Low prediction difference \\
Test accuracy gap (\%) & 0.37 & Minimal utility loss \\
MIA success rate & 0.499 & Effective privacy restoration \\
\bottomrule
\end{tabular}
\end{table}

\section{Limitations and Future Directions}
\label{app:limitations}

\textbf{CF Metric Instability.} The CF metric has high variance when accuracy differences are small. The denominator $|\text{Acc}_{\text{trained}} - \text{Acc}_{\text{retrain}}|$ in our experiments is 0.97\%, causing sensitivity to small fluctuations. Prediction agreement rate, KL divergence, or functional distance measures may provide more stable faithfulness assessment.

\textbf{Approximation Quality.} HF-KCU uses $k=10$ CG iterations. Tighter theoretical bounds on approximation error as a function of $k$ would strengthen the method's guarantees.

\textbf{Sequential Unlearning.} Repeated unlearning may accumulate approximation errors. Studying HF-KCU under multiple sequential deletions is important for production systems.

\textbf{Heterogeneous Systems.} Real federated systems exhibit device heterogeneity in compute, memory, and network. Adaptive strategies accounting for client-specific constraints warrant investigation.

\textbf{Beyond Vision.} Extending HF-KCU to transformers, large language models, and graph neural networks presents challenges due to different loss surfaces and parameter scales.

\textbf{Memory Constraints for Ultra-Large Models.} Each CG iteration stores vectors of size $d$; maintaining $k$ such vectors costs $O(kd)$ memory. For $d \sim 10^9$, even $k=10$ may exceed GPU memory. Memory-efficient variants (truncated Krylov methods, low-rank influence approximations) are future work.

Our adversarial analysis assumes bounded perturbations. We do not claim robustness against fully unconstrained adaptive poisoning attacks, strategic multi-client collusion, or minimax-optimal adversaries optimized against the unlearning mechanism. Extending to such settings would require a game-theoretic formulation with dedicated attack evaluations.

HF-KCU shows that federated systems can support practical data deletion without excessive computational costs. It complies with privacy regulations while preserving model utility.

\end{document}